\documentclass{article}

\usepackage{arxiv}
\usepackage[utf8]{inputenc}

\usepackage[english]{babel}
\usepackage[T1]{fontenc}
\usepackage{polski}
\usepackage{amsmath}
\usepackage{amsthm}
\usepackage{amssymb}
\usepackage{hyperref}

\usepackage[round]{natbib}

\newtheorem{example}{Example}
\newtheorem{theorem}{Theorem}
\newtheorem{proposition}{Proposition}

\newtheorem{corollary}{Corollary}
\newtheorem{definition}{Definition}
\newtheorem{notation}{Notation}

\author{Ismaïl Baaj \\ Univ. Artois, CNRS, CRIL, F-62300 Lens, France \\  \href{baaj@cril.fr}{baaj@cril.fr}}

\title{Chebyshev distances associated to the second members of systems  of max-product/Lukasiewicz  fuzzy relational equations}

\hypersetup{
pdftitle={Chebyshev distances associated to the second members of systems  of max-product/Lukasiewicz  fuzzy relational equations},
pdfsubject={cs.AI},
pdfauthor={Ismaïl Baaj},
pdfkeywords={Fuzzy set theory ; Systems of fuzzy relational equations ; Chebyshev approximations},
}
\date{}

\begin{document}

\maketitle

\begin{abstract}
In this article, we study the inconsistency of a system of $\max$-product fuzzy relational equations  and  of a system of $\max$-Lukasiewicz fuzzy relational equations.
For a system of $\max-\min$ fuzzy relational equations $A \Box_{\min}^{\max} x = b$ and using the $L_\infty$ norm,    \citep{arxiv.2301.06141} showed that
the Chebyshev distance $\Delta = \inf_{c \in \mathcal{C}} \Vert b - c \Vert$, where $\mathcal{C}$ is the set of second members of consistent systems defined with the same matrix $A$, can be computed by an explicit analytical formula according to the components of the matrix $A$ and its second member $b$.  In this article, we give analytical formulas analogous to that of \citep{arxiv.2301.06141}  to compute the Chebyshev distance associated to the second member of a system of $\max$-product fuzzy relational equations and that associated to the second member of a system of $\max$-Lukasiewicz fuzzy relational equations.
\end{abstract}


\section{Introduction}

Artificial Intelligence (AI) applications based on systems of fuzzy relational equations emerged thanks to \citep{sanchez1976resolution,sanchez1977}'s seminal work on solving systems of $\max-\min$ fuzzy relational equations. \citep{sanchez1976resolution}  gave necessary and sufficient conditions for a system to be consistent i.e., to have solutions. \citep{sanchez1977} showed that if the system is consistent, it has a greater solution and many minimal solutions, and he then described the complete set of solutions of the system.

However, the inconsistency of these systems remains difficult to address. While the majority of the approaches have investigated how to determine approximate solutions of an inconsistent system \citep{ di2013fuzzy,gottwald1986characterizations,klir1994approximate, pedrycz1983numerical,pedrycz1990algorithms,wangming1986fuzzy,WEN2022,wu2021analytical,wu2022analytical,xiao2019linear}, Pedrycz highlighted another strategy in \citep{pedrycz1990inverse}.  Given an inconsistent system, Pedrycz proposed to slightly modify its second member to obtain a consistent system. Some authors proposed algorithms for his procedure \citep{cuninghame1995residuation,li2010chebyshev}. In a recent preprint \citep{arxiv.2301.06141},  for an inconsistent system of $\max-\min$ fuzzy relational equations $A \Box_{\min}^{\max} x = b$, the author gave an explicit analytical formula to compute, using the $L_\infty$ norm, the Chebyshev distance $\Delta = \inf_{c \in \mathcal{C}} \Vert b - c \Vert$ where $\mathcal{C}$ is the set of second members of consistent systems defined with the same matrix $A$. As a Chebyshev approximation of the second member $b$ is a vector $c$ such that $\Vert b - c \Vert = \Delta$ and the system $A \Box_{\min}^{\max} x = c$ is a consistent system \citep{cuninghame1995residuation,li2010chebyshev}, the  formula of $\Delta$ led  the author of \citep{arxiv.2301.06141} to give the complete description of the structure of the set of Chebyshev approximations of  $b$. He then  described the approximate solutions set of the  considered inconsistent system, which is the set of solutions of  consistent systems $A \Box_{\min}^{\max} x = c$, where $c$ is a  Chebyshev approximation of $b$. \\
In our opinion, it would be relevant to study how these results can be extended to systems of $\max-T$ fuzzy relational equations, where $T$ is a continuous t-norm.
As it is well-known, many applications are based on systems of fuzzy relational equations  with the max-product composition or the max-Lukasiewicz composition e.g., \citep{di2013fuzzy,hirota1999fuzzy,nobuhara2006motion}.

\noindent  In this article, we tackle the problem of the inconsistency of a system of max-product fuzzy relational equations
and of a system of max-Lukasiewicz fuzzy relational equations. 
 We extend to systems of $\max-T$ fuzzy relational equations, where $T$ is a continuous t-norm, the definition and the properties  of the idempotent application $F$, see (\ref{eq:F}), introduced in \citep{arxiv.2301.06141}. This application was initially proposed as a reformulation of \citep{sanchez1976resolution}'s result  for  checking if a system of $\max-\min$ fuzzy relational equations defined with a fixed matrix and a given vector used as second member
is a consistent system. The idempotence and the right-continuity of the application $F$ lets us extend a fundamental result given for $\max-\min$ composition (Theorem 1 of \citep{cuninghame1995residuation}) to $\max-T$ composition (Theorem \ref{th:1}), and to obtain the greatest Chebyshev approximation of the second member of such a system (Corollary \ref{coro:1}). \\
Then, we give an explicit analytical formula to compute the Chebyshev distance associated to the second member of a system of $\max$-product fuzzy relational equations (Theorem \ref{th:2}). In the case of a system of $\max$-Lukasiewicz fuzzy relational equations, an analogous analytical formula is also given (Theorem \ref{th:3}). Each Chebyshev distance is computed according to the components of the matrix and the second member of the considered system.  

\noindent The article is structured as follows. In (Section \ref{sec:background}), we begin by reminding the necessary background on solving a system of $\max-T$ fuzzy relational equations where $T$ is a continuous t-norm. We extend the idempotent application $F$ of \citep{arxiv.2301.06141} to a system of $\max-T$ fuzzy relational equations.   
We remind the definition of the Chebyshev distance associated to the second member of a system of $\max-T$ fuzzy relational equations and extend to such a system a  fundamental result given for $\max-\min$ composition (Theorem 1 of   \citep{cuninghame1995residuation}).
In (Section \ref{sec:prelim}), we study the solving of two inequalities, which are involved in the computation of the Chebyshev distance associated to the second member of a system of $\max$-product fuzzy relational equations and that associated to the second member of a system of $\max$-Lukasiewicz fuzzy relational equations. In (Section \ref{sec:formulas}), we give the explicit analytical formulas to compute the two Chebyshev distances. 
Finally, we conclude with some perspectives. 

\noindent Our results are illustrated with examples (note that in all examples, the numbers are rounded to two decimal places).

\section{Background}
\label{sec:background}
Throughout this section, $T$ denotes a continuous t-norm.

\noindent
In what follows, we begin by reminding the definition and properties of a continuous t-norm   and its
 associated residual implicator.
To the t-norm defined by the usual product, we make explicit its associated residual implicator, as well as the one associated to the t-norm of Lukasiewicz.  \\
We then extend to systems of $\max-T$ fuzzy relational equations the idempotent application $F$  introduced in \citep{arxiv.2301.06141} for checking if a system of  $\max-\min$ fuzzy relational
equations defined with a fixed matrix and a given vector used as second member
is a consistent system. Finally, we remind the definition of the Chebyshev distance associated to the second member of a system of $\max-T$ fuzzy relational equations and extend to such a system  a  fundamental result  proven in  \citep{cuninghame1995residuation} for $\max-\min$ composition. 

\subsection{Notations}
The following notations were given in \citep{arxiv.2301.06141} for the case of a system of $\max-\min$ fuzzy relational equations and we reuse them in this article.

\noindent
 By $[0,1]^{n\times m}$ we denote the set of matrices of size $(n,m)$ i.e., $n$ rows and $m$ columns, whose components are in $[0,1]$. In particular:
\begin{itemize}
    \item $[0,1]^{n\times 1}$ denotes the set of column vectors of $n$ components,
    \item $[0,1]^{1\times m}$ denotes  the set of row matrices of $m$ components. 
\end{itemize}

\noindent On the set $[0,1]^{n\times m}$, we use the order relation $\leq$ defined by:
\[ A \leq B \quad \text{iff we have} \quad  a_{ij} \leq b_{ij} \quad \text{ for all } \quad 1 \leq i \leq n, 1 \leq j \leq m,    \]
\noindent where $A=[a_{ij}]_{1 \leq i \leq n, 1 \leq j \leq m}$ and $B=[b_{ij}]_{1 \leq i \leq n, 1 \leq j \leq m}$. 

For $x,y,z,u,\delta \in [0,1]$, we put:
\begin{itemize}
    \item $x^+ = \max(x,0)$,
    \item $\overline{z}(\delta) = \min(z+\delta,1)$, 
    \item $\underline{z}(\delta) = \max(z-\delta,0) = (z-\delta)^+$
\end{itemize}
\noindent and we have the following equivalence in $[0,1]$: 
\begin{equation}\label{ineq:xyxbarybar}
    \mid x - y \mid \leq \delta \Longleftrightarrow \underline{x}(\delta) \leq y \leq \overline{x}(\delta).
\end{equation}

\noindent For our work, to a column-vector $b = [b_i]_{1 \leq  i \leq n}$ and a number $\delta \in [0,1]$, we associate two column-vectors:
    \begin{equation}\label{def:bstarhautbas} 
  \underline{b}(\delta) = [(b_i  - \delta)^+]_{1 \leq  i \leq n} \quad \text{and} \quad \overline{b}(\delta) = [\min(b_i  + \delta, 1)]_{1 \leq  i \leq n}.\end{equation}
\noindent These vectors $\underline{b}(\delta)$ and $\overline{b}(\delta)$ were already introduced e.g.,  in  \citep{cuninghame1995residuation} (with others notations) and in \citep{li2010chebyshev}. \\
\noindent Then, from (\ref{ineq:xyxbarybar}), we deduce for any $c = [c_i]_{1\leq i \leq n} \in [0,1]^{n \times 1}$: 
\begin{equation}\label{ineq:bbbar}
    \Vert b - c \Vert \leq \delta \Longleftrightarrow \underline{b}(\delta) \leq c \leq \overline{b}(\delta).
\end{equation}
\noindent where $ \Vert b - c \Vert = \max_{1 \leq i \leq n}\mid b_i-c_i\mid$.
\subsection{T-norms and their associated residual implicators }

\noindent
A triangular-norm  (t-norm, see  \citep{klement2013Fulltriangular}) is a map  $T: [0,1] \times [0,1] \mapsto [0,1]$, which satisfies:
\begin{itemize}
    \item[] $T$ is commutative:   $T(x,y) = T(y,x)$, 
    \item[] $T$ is associative:  $T(x, T(y, z)) = T(T(x, y), z)$, 
    \item[] $T$ is increasing with respect to the second variable: $y \leq z \, \Longrightarrow \, T(x, y) \leq T(x, z)$,
    \item[]$T$ has $1$ as neutral element: $T(x, 1) = x$.
\end{itemize}

\noindent The residual implicator associated to the t-norm $T$ is the map 
${\cal I}_T : [0, 1 ] \times [0, 1] \rightarrow [0, 1 ] : (x, y) \mapsto {\cal I}_T(x, y) = \sup\{z \in [0, 1]\,\mid\, T(x, z) \leq y\}$.

\noindent
The main properties of the residual implicator ${\cal I}_T$ of a continuous t-norm $T$ are:  for all $a, b   \in [0, 1]$, we have: 
\begin{itemize}
    \item   
${\cal I}_T(a, b) = \max\{z \in [0, 1]\,\mid\, T(a, z) \leq b\}$. Therefore, $T(a, {\cal I}_T(a, b)) \leq b$.
\item
  ${\cal I}_T$ is left-continuous and decreasing in its first argument as well as right-continuous and increasing in its second argument.
  \item
  For all $z\in [0, 1]$, we  have: 
  $$T(a, z) \leq b \Longleftrightarrow z \leq {\cal I}_T(a, b).$$
  \item We have 
  $\,b \leq {\cal I}_T(a, T(a, b))$.
\end{itemize}

The t-norm defined by the usual product   is denoted by $T_P$. Its associated residual implicator is the Goguen product, and we have:
\begin{equation}\label{eq:tnormproduct}
    T_P(x,y) = x \cdot y \quad ;  \quad {\cal I}_{T_P}(x,y) = x \underset{GG}{\longrightarrow} y = \begin{cases}1 & \text{ if } x \leq y \\ \frac{y}{x} &\text{ if } x > y \end{cases}.
\end{equation}
Lukasiewicz's t-norm is denoted by $T_L$ and we have:
\begin{equation}\label{eq:tnormluka}
    T_{L}(x,y)= \max(x + y - 1, 0) = {(x + y - 1)}^{+} \quad ; \quad 
 {\cal I}_{T_L}(x,y)  = x \underset{L}{\longrightarrow} y = \min(1-x+y,1).
\end{equation}
In the rest of this section, we set a continuous t-norm denoted $T$.

\subsection{Solving systems of \texorpdfstring{$\max-T$}{max-T} fuzzy relational equations}

\noindent Let $A=[a_{ij}] \in [0,1]^{n\times m}$ be a matrix of size $(n,m)$ and $b=[b_{i}] \in [0,1]^{n\times 1}$ be a vector of $n$ components. The system of $\max-T$ fuzzy relational equations  associated to $(A,b)$  is of the form:
\begin{equation}\label{eq:sys}
    (S): A \Box_{T}^{\max} x = b,
\end{equation}
\noindent where $x = [x_j]_{1 \leq j \leq m} \in [0,1]^{m\times 1}$ is  an unknown vector of $m$ components and the operator $\Box_{T}^{\max}$ is the matrix product that uses the continuous  t-norm $T$ as the product and $\max$  as the addition. Equivalently, the system can also be written as:
\begin{equation*}
    \max_{1 \leq j \leq m} T(a_{ij},x_j) = b_i,\, \forall i \in \{1, 2,\dots, n\}.
\end{equation*}
\noindent  The studies on systems of fuzzy relational equations adopt two equivalent notation conventions, which differ in their representation of the unknown part and the second member: either as column vectors or as row vectors.  The transpose  map allows us to switch from one notation to the other.

\noindent To check if the system (S), see (\ref{eq:sys}), is consistent, we compute the following vector: \begin{equation}\label{eq:egretestsol}
e = A^t \Box_{{\cal I}_T}^{\min} b,\end{equation}
\noindent where $A^t$ is the transpose of $A$ and the matrix product $\Box_{{\cal I}_T}^{\min}$ uses the residual implicator  ${\cal I}_T$ (associated to $T$) as the product and $\min$ as the addition.
The vector $e$ is the potential maximal solution of the system $(S)$. 

\noindent Thanks to \citep{sanchez1976resolution}'s seminal work  on the solving of systems $\max-\min$ fuzzy relational equations, and \citep{pedrycz1982fuzzy,pedrycz1985generalized} and \citep{MIYAKOSHI198553} (using other t-norms than $\min$), we have the following equivalence:
\begin{equation}\label{eq:consiste}
    A \Box_T^{\max} x = b \text{ is consistent}\Longleftrightarrow A \Box_T^{\max} e = b. 
\end{equation}
In the rest of this subsection, we will give   (Proposition \ref{proposition:appF}) and show that the proof of its second statement implies the equivalence (\ref{eq:consiste}). \\
\noindent  The set of solutions of the system $(S)$ is denoted by:
    \begin{equation}\label{eq:setofsolutionsS}
{\cal  S  } = {\cal  S  }(A, b) = \{ v \in [0, 1]^{m \times 1} \,\mid \,  A \, \Box_{T}^{\max} v  = b \}. 
\end{equation}
\noindent The structure of the solution set was described by \citep{sanchez1977} (with the $\max-\min$ composition) and \citep{di1982solution,di1984fuzzy,di2013fuzzy}.

\begin{example}
Let: 
\begin{equation}\label{eq:ex:Ab}
    A = \begin{bmatrix}
    0.31	&0.49	&0.76\\
0.34	&0.9	&0.15\\
0.94	&0.47	&0.05\\
    \end{bmatrix} \text{ and } b = \begin{bmatrix}
    0.73\\ 0.84\\ 0.61
    \end{bmatrix}.
\end{equation}
To check if the  systems $A \Box_{T_P}^{\max} x = b$ and $A \Box_{T_L}^{\max} x = b$ are consistent,  we compute their respective potential greatest solution:
\begin{equation*} 
    e_{T_P} = A^t \Box_{{\cal I}_{T_P}}^{\min} b = \begin{bmatrix} 0.65\\0.93\\ 0.96 \end{bmatrix} \quad \text{ and } \quad e_{T_L} = A^t \Box_{{\cal I}_{T_L}}^{\min} b = \begin{bmatrix} 0.67 \\ 0.94\\ 0.97 \end{bmatrix}.
\end{equation*}
\noindent These two systems are consistent because:
\begin{equation*}
     A \Box_{T_P}^{\max} e_{T_P} = A \Box_{T_L}^{\max} e_{T_L} = b.
    \end{equation*}
\end{example}

\noindent In \citep{arxiv.2301.06141}, the author introduced an idempotent application denoted $F$,  to check if a system of $\max-\min$ fuzzy relational equations defined with a fixed matrix  $A =[a_{ij}]_{1 \leq i \leq n, 1 \leq j \leq m} \in [0,1]^{n\times m}$ and a given vector $c \in [0,1]^{n \times 1}$  used as second member is a consistent system.  The application is a reformulation of \citep{sanchez1976resolution}'s result. For systems of $\max-T$ fuzzy relational equations which uses a continuous t-norm  $T$, we extend the definition of the application $F$ as follows:
\begin{equation}\label{eq:F} 
    F :  [0, 1]^{n \times 1} \longrightarrow [0, 1]^{n \times 1} : c =[c_i] \mapsto F(c) =[F(c)_i]
\end{equation}
\noindent where:
\begin{equation}\forall i \in \{1, 2, \dots, n\},\, F(c)_i = \max_{1 \leq j \leq m} T(a_{ij}, \min_{1 \leq k \leq n} 
{\cal I}_T(a_{kj}, c_k)).\end{equation}

The properties of the application $F$ given in \citep{arxiv.2301.06141} for a system of  $\max-\min$ fuzzy relational equations are valid for systems of $\max-T$ fuzzy relational equations:  

\begin{proposition}\label{proposition:appF}
\mbox{}
\begin{enumerate}
\item $\forall c \in [0,1]^{n \times 1}$, $F(c) \leq c$.
\item For any  vector $c \in [0,1]^{n \times 1}$ we have:
\[ F(c) = c \Longleftrightarrow \text{ the system } A  \Box_{T}^{\max} x = c \text{ is consistent.} \]
 
    \item $F$ is idempotent i.e., $\forall c \in [0,1]^{n \times 1}, F(F(c)) = F(c)$. 

    \item  $F$ is increasing and right-continuous.
\end{enumerate}
 \end{proposition}
 \noindent The application $F$ being right-continuous  at a point $c\in [0,1]^{n \times 1}$ means: for any sequence $(c^{(k)})$  in  $[0,1]^{n \times 1}$ such that  $(c^{(k)})$ converges to $c$ when  $k \rightarrow \infty$ and verifying $\forall k, c^{(k)} \geq c$, we have:  
\begin{center}
    $F(c^{(k)}) \rightarrow F(c)$ when $k \rightarrow \infty$.
\end{center}
\begin{proof}
  We adapt the proofs given in \citep{arxiv.2301.06141} for the $\max-\min$ composition  to the $\max-T$ composition  in the following way:
  
\noindent
The first assertion is deduced from the inequality 
$T(a , {\cal I}_T(a , b)) \leq b$.

\noindent For the second statement, if we suppose that $F(c) = c$, it means that  $e = A^t \Box_{{\cal I}_T}^{\min} c$ is a solution of the system $A \Box_T^{\max } x = c$ (so it is consistent). \\ 
For the proof of the implication $\Longleftarrow$, we use the general inequality: 
$$\forall x\in [0 , 1]^{m \times 1},\,x  \leq A^t \Box_{{\cal I}_T}^{\min} (A\Box_T^{\max }x), $$
which is deduced from the inequality $b \leq {\cal I}_T(a , T(a , b))$. We also obtain that $e = A^t \Box_{{\cal I}_T}^{\min} c$ is the greatest solution of the system $A \Box_T^{\max } x = c$. Clearly, we have also proven the equivalence (\ref{eq:consiste}).

\noindent
Thanks to the two assertions above, the idempotence of the application $F$ is proved as in the case of the 
$\max-\min$ composition.

\noindent
It is easy to prove that the application $F$ is increasing. The right-continuity   of the residual implicator ${\cal I}_T$ in its second argument implies the right-continuity of the application $F$.
 
\end{proof}

\begin{example}
(continued)
Let $c = \begin{bmatrix} 0.74\\ 0.66\\ 0.62\end{bmatrix}$. When we use any of the two t-norms $T_P$ or $T_L$, we have the equality $F(c) = c$. So the systems  $A \Box_{T_P}^{\max}x= c$ and $A \Box_{T_L}^{\max}x= c$ are consistent. 
\end{example}

\subsection{Chebyshev distance associated to the second member of a system of \texorpdfstring{$\max-T$}{max-T} fuzzy relational equations}

\noindent To the matrix $A$ and the second member $b$ of the system $(S)$ of $\max-T$ fuzzy relational equations, see (\ref{eq:sys}), let us  associate  the set of  vectors $c = [c_i] \in [0,1]^{n \times 1}$ such that the system $A \Box_{T}^{\max} x = c$ is consistent:
\begin{equation}\label{def:setofsecondmembersB}
    \mathcal{C} = \{ c = [c_i] \in {[0,1]}^{n \times 1} \mid  A \Box_{T}^{\max} x = c \text{ is consistent} \}.
\end{equation}
\noindent This set allows us to define the Chebyshev distance associated to the second member $b$ of the system $(S)$.
 \begin{definition}\label{def:chebyshevdist}
The Chebyshev distance associated to the second member $b$ of the system $(S): A \Box_{T}^{\max}x = b$ is: 
\begin{equation}\label{eq:delta}
    \Delta = \Delta(A,b) =  \inf_{c \in \mathcal{C}} \Vert b - c \Vert 
    \end{equation}

    \end{definition}
\noindent where:
\[ \Vert b - c \Vert = \max_{1 \leq i \leq n}\mid b_i - c_i\mid.\]

We have the following fundamental result,  proven for $\max-\min$ composition in \citep{cuninghame1995residuation}, that we extend for $\max-T$ composition:    
\begin{theorem}\label{th:1}  

\begin{equation}\label{eq:deltaCUNINGH}
    \Delta =        \min\{\delta\in [0, 1] \mid \underline b(\delta) \leq F(\overline b(\delta))\}.
\end{equation}

\end{theorem}
 The proof of (Theorem \ref{th:1}) is given in three steps.
\begin{proof}
\noindent We put $\Delta'  = \inf\{\delta\in [0, 1] \,| \, \underline b(\delta) \leq F(\overline b(\delta))\}$. 

\noindent $\bullet$ Step 1. 
We want to prove that we have $\Delta' = \min\{\delta\in [0, 1] \mid \underline b(\delta) \leq F(\overline b(\delta))\} $. Indeed, as    we have 
$\underline b(1)  = 0 \leq F(\overline b(1))$, the set $\{ \delta\in [0, 1] \mid \underline b(\delta) \leq F(\overline b(\delta))\}$ is non-empty. Therefore, 
 we can find a sequence $(\delta_k)$ in $[0, 1]$ verifying:
\begin{enumerate}
\item $\forall k,\, \underline b(\delta_k) \leq F(\overline b(\delta_k))$, 
\item $\forall k,\, \delta_k \geq \Delta'$,
\item $ \delta_k \rightarrow \Delta'$.
\end{enumerate}
As $\forall k, \overline{b}(\delta_k)\geq \overline{b}(\Delta')$ and $\overline{b}(\delta_k)  \rightarrow \overline{b}(\Delta')$, the right-continuity of the application $F$ at $\overline{b}(\Delta')$ implies: 
$$F(\overline b(\delta_k)) \rightarrow F(\overline b(\Delta')).$$
The inequalities $\underline b(\delta_k) \leq F(\overline b(\delta_k))$ lead, by the passage to the limit, to: 
 $$\underline b(\Delta') \leq F(\overline b(\Delta')).$$
So $\Delta' = \min\{\delta\in [0, 1] \mid \underline b(\delta) \leq F(\overline b(\delta))\}.$

\noindent $\bullet$ Step 2. We want to prove that $\Delta\leq \Delta'$. 

\noindent
We have:
$$\underline b(\Delta') \leq F(\overline b(\Delta'))  \leq  \overline b(\Delta').$$
By (\ref{ineq:bbbar}) and by noticing that  $F(\overline b(\Delta')) \in {\cal C}$, we deduce:
$$\Delta \leq \Vert b - F(\overline b(\Delta')) \Vert \leq \Delta'.$$

\noindent $\bullet$ Step 3. We want to prove that $\Delta' \leq \Delta$. For this it is sufficient to show that:  
$$\forall c \in {\cal C},\, \Delta' \leq \Vert  b - c\Vert.$$
Indeed, for any $  c \in {\cal C}$, we put $\delta =  \Vert  b - c\Vert$, and we have by (\ref{ineq:bbbar}):
$$\underline b(\delta)  \leq c \leq  \overline b(\delta).$$
Therefore  $F(c)  = c$  and the growth of $F$ lead to: 
$$\underline b(\delta) \leq c = F(c)  \leq F(\overline b(\delta)).$$ So:
$$\Delta' \leq \delta =  \Vert  b - c\Vert.$$
We have proved that $\Delta = \Delta' = \min\{\delta\in [0, 1] \mid \underline b(\delta) \leq F(\overline b(\delta))\}$.

\end{proof}
We define the set of Chebyshev approximations of $b$:
\begin{definition}
   The set of Chebyshev approximations of $b$ is defined using the set $\mathcal{C}$, see (\ref{def:setofsecondmembersB}), and the Chebyshev distance associated to $b$ (Definition \ref{def:chebyshevdist}): 
\begin{equation}\label{def:chebyshevsetapproxB} 
{\cal C}_{b}  = \{c \in {\cal C} \,\mid \,  \Vert b - c\Vert = \Delta(A ,  b)\}.
\end{equation} 
\end{definition}
We have:
\begin{corollary}\label{coro:1}
$F(\overline{b}(\Delta))\in {\cal C}_b$ is the greatest Chebyshev approximation of the second member $b$ of the system $(S)$. 
\end{corollary}
\begin{proof}
Let us first show that we have $ \Vert b - F(\overline b( \Delta))\Vert =  \Delta$. \\Indeed, from (Theorem \ref{th:1}) and (Proposition \ref{proposition:appF}),   we know that  $\underline b(\Delta) \leq F(\overline b(\Delta)) \leq \overline b(\Delta)$, so by (\ref{ineq:bbbar}), $\Vert b - F(\overline b( \Delta))\Vert \leq  \Delta $.

\noindent
 As
$F(\overline b( \Delta)) \in {\cal C}$ (Proposition \ref{proposition:appF}),   we know by (Definition \ref{def:chebyshevdist}): 
$$\Delta   \leq   
\Vert b - F(\overline b( \Delta)) \Vert.$$
We have proven  that 
$F(\overline b( \Delta)) \in {\cal C}_b$.

\noindent
To prove that $F(\overline b( \Delta))$ is the greatest Chebyshev approximation of the second member $b$ of the system $(S)$ we must prove that $c \leq F(\overline b( \Delta))$ for any $c\in {\cal C}_b$.\\
Let $c\in {\cal C}_b$. As 
 $\Delta   =  \Vert  b - c\Vert$,     by (\ref{ineq:bbbar}) we have: $$\underline b(\Delta)  \leq c \leq  \overline b(\Delta).$$
 By the growth of $F$, we deduce:
$$c = F(c) \leq F(\overline b( \Delta)).$$
\end{proof}
\noindent As a consequence, we have: 
\begin{corollary}
    \label{corollary:deltazero}
    $$\Delta = \min_{c \in \mathcal{C}} \Vert b - c \Vert.$$
   \[
       \Delta = 0 \Longleftrightarrow\text{ the system $(S)$ is consistent.} 
   \]
\end{corollary}
\noindent Therefore, $\Delta = 0$ is a  necessary and sufficient condition for the system $(S)$ to be consistent.

\section{Preliminaries computations}
\label{sec:prelim}
In this section,  we  solve two inequalities. The first one will be involved in the analytical formula to compute  the Chebyshev distance  associated to the second member of a system of $\max-T$ fuzzy relational equations when the t-norm is the usual product i.e., $T = T_P$, see (\ref{eq:tnormproduct}). The second one will be involved in the formula to compute the Chebyshev distance  when we use Lukasiewicz's t-norm i.e.,  $T = T_L$, see (\ref{eq:tnormluka}).  A similar study was carried out  in \citep{arxiv.2301.06141} for $\max-\min$ composition.

\subsection{\texorpdfstring{T-norm product $T = T_P$}{T-norm product T = TP}}
 \noindent  Let $x,y,z,u \in [0,1]$ be fixed.
\noindent  For the t-norm product $T_P$, let us study the following inequality that involves the Goguen product $\underset{GG}{\longrightarrow}$ (see (\ref{eq:tnormproduct})), for $\delta \in [0, 1]$:  
\begin{equation}\label{ineq:gen:tp}
 \underline{x}(\delta) \leq u \cdot (y \underset{GG}{\longrightarrow} \overline{z}(\delta))
\end{equation}
\noindent where $y \underset{GG}{\longrightarrow} \overline{z}(\delta) = \begin{cases}1 & \text{ if } y - z \leq \delta \\
\frac{z+ \delta}{y} & \text{ if } y -z > \delta\end{cases}$.\\
\textit{We look for the smallest value of $\delta$ so that the inequality (\ref{ineq:gen:tp}) is true}. \\
\noindent Let:
\begin{equation}\label{eq:sigmagg}
    \sigma_{GG}(u,x,y,z) = \max[(x-u)^+, \min(\varphi(u,x,y,z),(y-z)^+)]
\end{equation}
\noindent where $\varphi(u,x,y,z) = \begin{cases}\frac{(x \cdot y - u\cdot z)^+}{u+y} &\text{ if } u > 0 \\ x &\text{ if } u = 0\end{cases}$. 
\begin{proposition}\label{In3}
For all $\delta \in [0,1]$, we have:\\
\[\underline{x}(\delta) \leq u \cdot (y \underset{GG}{\longrightarrow} \overline{z}(\delta)) \Longleftrightarrow \sigma_{GG}(u,x,y,z) \leq \delta.\]
\end{proposition}
\begin{proof}\mbox{}\\
If  $u = 0$, then  $\sigma_{GG}(0, x, y, z)  =  x$  and we immediately get the desired equivalence:  
$$(x - \delta)^+ \leq   0 \Longleftrightarrow x \leq \delta \Longleftrightarrow \sigma_{GG}(0, x, y, z)  \leq \delta.$$
It remains for us to study the case where $u > 0$.

\noindent 
Suppose that  $ (x - \delta)^+ \leq  u(y  \underset{GG}{\longrightarrow} \overline z(\delta))$.

\begin{itemize}
    \item If  $y -  z \leq \delta$, we have: 
 $$(x - \delta)^+ \leq u (y  \underset{GG}{\longrightarrow} \overline z(\delta)) = u \cdot 1 = u.$$
 We know that the  inequality  $(x - \delta)^+ \leq u$  is equivalent to the inequality   $(x - u)^+ \leq \delta$    \citep{arxiv.2301.06141}. 
 As $\min(\varphi(u, x, y, z), (y - z)^+) \leq (y - z)^+ \leq \delta$, we obtain:  
 $$\sigma_{GG}(u, x, y, z)   = \max[(x - u)^+, \min(\varphi(u, x, y, z), (y - z)^+)] \leq \delta.$$
 \item If $y - z > \delta$, we have: 
$$(x - \delta)^+ \leq  u(y  \underset{GG}{\longrightarrow} \overline z(\delta)) = u \dfrac{ z +   \delta}{y} \leq u$$
so, $(x - u)^+ \leq \delta$   and moreover
\begin{align}
(x - \delta)^+ -   u(y  \underset{GG}{\longrightarrow} \overline z(\delta)) &= \max(x - \delta - \dfrac{uz + u \delta}{y}, -\dfrac{uz + u \delta}{y})\nonumber\\
& =  \max(\dfrac{xy - \delta y -  uz - u \delta}{y}, -\dfrac{uz + u \delta}{y})\leq 0.\nonumber \end{align}
\noindent
The last inequality is equivalent to   $xy - \delta y -  uz - u \delta   \leq 0$ and also: 
$$\dfrac{xy  -     uz}{u + y} \leq \delta.$$
\end{itemize}
Then:

\noindent

$\bullet$ If $(xy - u z)^+ = 0$, we have:
$$\varphi(u, x, y, z)  = \dfrac{(xy - u z)^+}{u + y} = 0$$
and also
$$\sigma_{GG}(u, x, y, z)    = \max[(x - u)^+, \min(\varphi(u, x, y, z), (y - z)^+) = (x -u)^+  \leq \delta.$$

\noindent
$\bullet$  If $(xy - u z)^+ >  0$, we have:
$$ \varphi(u, x, y, z)  = \dfrac{ xy - u z }{u + y}  \leq \delta$$
and also
\begin{align} \sigma_{GG}(u, x, y, z)    & = \max[(x - u)^+, \min(\varphi(u, x, y, z), (y - z)^+)]\nonumber\\
& =  \max[(x - u)^+,  \varphi(u, x, y, z)]\nonumber\\
&\leq   \delta.\nonumber\end{align}

\noindent Let us prove now the remaining implication.\\
Suppose that $\sigma_{GG}(u, x, y, z)  \leq \delta$:
\begin{itemize}
    \item If   $(x - \delta)^+ = 0$, we get trivially that $ (x - \delta)^+  = 0 \leq  u(y  \underset{GG}{\longrightarrow} \overline z(\delta))$.
    \item If $y  \underset{GG}{\longrightarrow} \overline z(\delta) = 1$, we must prove that  $ (x - \delta)^+   \leq  u$, which is equivalent to 
$(x - u)^+ \leq \delta$. We have: 
$$(x - u)^+ \leq \max[(x - u)^+, \min(\varphi(u, x, y, z), (y - z)^+] =  \sigma_{GG}(u, x, y, z)  \leq \delta.$$
\end{itemize}
\noindent It remains for us to study the case where $(x - \delta)^+  = x - \delta > 0$ and   
$y  \underset{GG}{\longrightarrow} \overline z(\delta) < 1$. 

\noindent From  $y  \underset{GG}{\longrightarrow} \overline z(\delta) < 1$, we deduce: 
$$y - z > \delta \quad \text{and} \quad y  \underset{GG}{\longrightarrow} \overline z(\delta) = \dfrac{z + \delta}{y}.$$
\noindent  Then, we have: 
\begin{align} (x - \delta)^+  - u(y  \underset{GG}{\longrightarrow} \overline z(\delta)) &=    x - \delta   - \dfrac{uz + u\delta}{y}\nonumber\\
& =   \dfrac{u + y}{y} (\dfrac{xy -     uz  }{u + y} - \delta).\nonumber\end{align}
To obtain $(x - \delta)^+  - u(y  \underset{GG}{\longrightarrow} \overline z(\delta))\leq 0$, it is sufficient to prove that $$\dfrac{xy -     uz  }{u + y}  \leq \delta.$$

\begin{itemize}
    \item If   $ xy  \leq      uz $, we have $\dfrac{xy -     uz  }{u + y} \leq 0 \leq \delta$.
    \item If  
  $xy  > uz$ and taking into account the inequality $y - z > \delta$, we deduce:  
$$\delta  \geq \sigma_{GG}(u, x, y, z)  
  \geq \min(\dfrac{(xy -     uz)^+  }{u + y}, y -z)   = \dfrac{xy -     uz  }{u + y}.$$ 
\end{itemize}
\end{proof}
\noindent
From (Proposition \ref{In3}), we immediately deduce:
\begin{theorem}\label{theoremedeGege}
For any $x,y,z,u \in [0,1]$, we have:
    \begin{equation}\label{ineq:P2}
    \sigma_{GG}(u, x, y, z) =    \min\{\delta\in [0, 1] \,\mid\, \underline{x}(\delta) \leq u \cdot (y \underset{GG}{\longrightarrow} \overline{z}(\delta)) \}.
\end{equation}
\end{theorem}

\noindent We illustrate this result:
\begin{example}
Let $x = 0.4$,
    $y = 0.3$,
    $z = 0.6$ and 
    $u = 0.2$. We want  to obtain the smallest value of $\delta \in [0,1]$ so that $ \underline{x}(\delta) \leq u \cdot (y \underset{GG}{\longrightarrow} \overline{z}(\delta))$ is true. We have $u \cdot ( y \underset{GG}{\longrightarrow}  z)  = 0.2$.\\

We compute $\varphi(u,x,y,z)=\frac{(0.4 \cdot 0.3 - 0.2\cdot 0.6)^+}{0.2+0.3}=0.$
    \begin{align*}
    \delta &\!\begin{aligned}[t]
    &= \sigma_{GG}(u,x,y,z) \\
    &= \max[(0.4-0.2)^+, \min(0,(0.3-0.6)^+)]\\
     &= 0.2.
     \end{aligned}
     \end{align*}
\noindent We have $\underline{x}(\delta)=x - 0.2 = 0.2$ and $\overline{z}(\delta)=z+0.2 = 0.8$. Therefore, $u \cdot (y \underset{GG}{\longrightarrow} \overline{z}(\delta))= 0.2$ and:
\begin{equation*}
   \underline{x}(\delta) = u \cdot (y \underset{GG}{\longrightarrow} \overline{z}(\delta)).
\end{equation*}
\end{example}
\subsection{\texorpdfstring{Lukasiewicz's t-norm $T = T_L$}{Lukasiewicz's t-norm T = TL}}
For Lukasiewicz's t-norm $T_L$, we study the following  inequality that involves Lukasiewicz's implication $\underset{L}{\longrightarrow}$ (see (\ref{eq:tnormluka})) for  $\delta \in [0, 1]$: 
\begin{equation}\label{ineq:gen:tl}
     \underline{x}(\delta) \leq \max(0, \,y \underset{L}{\longrightarrow} \overline{z}(\delta) - u),
\end{equation}
\noindent where $y \underset{L}{\longrightarrow} \overline{z}(\delta) = \begin{cases}1 & \text{ if } y - z \leq \delta \\
1 - y + z + \delta  & \text{ if } y -z > \delta\end{cases}$.\\
\noindent Let $v = x+u-1$ and:
\begin{equation}\label{eq:sigmaL}
  \sigma_L(u,x,y,z) = \min(x, \max(v^+, \frac{(v + y - z )^+}{2})).
\end{equation}
\noindent Then:
\begin{proposition}\label{prop4}
For all $\delta\in[0, 1]$, we have:
 \begin{equation}\label{ineq:gen:tlL}\underline{x}(\delta) \leq \max(0, y \underset{L}{\longrightarrow} \overline{z}(\delta) - u) 
 \,
 \Longleftrightarrow \,
 \sigma_L(u,x,y,z) \leq \delta.\end{equation}
 
\end{proposition}
 
\begin{proof}
We  remark that for $x \leq \delta$ the two assertions of the equivalence are verified: we have 
$(x - \delta)^+ = 0$ and $\sigma_{L}(u, x, y, z) \leq x \leq \delta$. It remains for us to show the equivalence in the case where $x > \delta$.

\noindent
Set $f(u, x) = (x + u - 1)^+, \, g(u, x, y, z)  =  \dfrac{(x + u - 1 + y - z )^+}{2}$. Then, we have:
$$\sigma_{L}(u, x, y, z) = \min(x, \max(f(u, x), g(u, x, y, z))).$$
Let us prove that 
$\underline{x}(\delta) \leq \max(0, y \underset{L}{\longrightarrow} \overline{z}(\delta) - u) 
 \,
 \Longrightarrow \,
 \sigma_L(u,x,y,z) \leq \delta$.

 \noindent
As we suppose that $ x > \delta$, we must prove:
$$\max(f(u, x), g(u, x, y, z) )\leq \delta.$$
This last inequality will be proven in the following two steps:
\begin{enumerate}
    \item  If  $y - z \leq  \delta $, then  $ y  \underset{L}{\longrightarrow} \overline z(\delta) - u = 1 - u$ and we have: 
$$x - \delta \leq \max(0,   1 - u)$$
which we put in the form  
\begin{align}x - \delta  -  \max(0,   1 - u) & = x - \delta  + \min(0, u - 1)\nonumber\\
&= \min(x - \delta, x + u -1  - \delta)   \nonumber\\
& \leq 0. \nonumber\end{align}
As $x - \delta > 0$, we deduce that  $x + u - 1 \leq \delta$ and also:  
$$f(u, x)  = (x + u - 1)^+ \leq \delta.$$
Let us show that we also have: 
$$g(u, x, y, z)  =  \dfrac{(x + u - 1 + y - z)^+}{2} \leq \delta.$$
Indeed, if  $ g(u, x, y, z) = 0$, we have trivially  $g(u, x, y, z) = 0 \leq \delta$.

\noindent
If $ g(u, x, y, z) > 0$, then we have: 
\begin{align}
 g(u, x, y, z) & = \dfrac{ x + u - 1 + y - z }{2}\nonumber\\
&\leq  \dfrac{ x + u - 1 +   \delta }{2}  \nonumber\\
&=  \dfrac{x + u -1}{2} + \dfrac{\delta}{2} \nonumber\\
& \leq  \dfrac{(x + u -1)^+}{2} + \dfrac{\delta}{2} \quad \text{(we already know that $f(u,x) =(x + u -1)^+ \leq \delta$)} \nonumber\\
& \leq \dfrac{\delta}{2} +  \dfrac{\delta}{2}  = \delta.
\nonumber\end{align}
\noindent In summary, we have shown  $y - z \leq \delta  \,\Longrightarrow \, \max(f(u, x), g(u, x, y, z) )\leq \delta$.
\item If  $y - z  >  \delta $, then 
   $ y  \underset{L}{\longrightarrow} \overline z(\delta) - u = 1 - y + z + \delta - u$ and we have: 
$$ x - \delta \leq \max(0,   1 - y + z + \delta - u).$$
We rewrite this last inequality in the form:  
\begin{align}x - \delta  -  \max(0,   1 - y + z + \delta - u) & = x - \delta  + \min(0, u - 1 + y - z     - \delta)\nonumber\\
&= \min(x - \delta,\,\, x + u - 1   + y - z    - 2 \delta)  \nonumber\\
& \leq 0. \nonumber\end{align}
As  $x - \delta > 0$, we deduce $x + u - 1   + y - z  \leq 2\delta$ and also:
$$g(u, x, y, z)   = \dfrac{ (x + y - z + u - 1)^+ }{2} \leq \delta.$$
It remains for us to show that $f(u, x) = (x + u - 1)^+ \leq \delta$. \\In fact, if $f(u, x) > 0$, then we have:  
\begin{align}
  \dfrac{  x + u - 1 + y - z   }{2}    & >     \dfrac{  x + u - 1 +  \delta  }{2}  \nonumber\\
&=  \dfrac{ x +   u - 1 }{2} +  \dfrac{\delta}{2} \nonumber\\
 & = \dfrac{ f(u, x)}{2} + \dfrac{\delta}{2}.
 \nonumber\end{align}
We deduce that  $x + u - 1 + y - z  > 0$, and also: 
$$g(u, x, y, z) = \dfrac{  x + u - 1 + y - z }{2}   > \dfrac{ f(u, x)}{2} + \dfrac{\delta}{2}.$$
As we have already shown that $ g(u, x, y, z) \leq \delta $, we finally obtain: 
$$f(u, x) < 2\,  g(u, x, y, z) -  \delta  \leq 2 \delta - \delta = \delta.$$
In summary, we have shown  $y - z >  \delta  \,\Longrightarrow \, \max(f(u, x), g(u, x, y, z) )\leq \delta$.
\end{enumerate}

\noindent We have completed the proof of the implication $\Longrightarrow$.

\noindent Let us  assume now that: $$\sigma_{L}(u, x, y, z)  = 
 \min(x, \max((x + u - 1)^+, \dfrac{(x + u - 1 + y - z )^+}{2})) \leq \delta$$ 
 and let us show that: 
$$(x - \delta)^+ \leq  \max(0, y  \rightarrow_{L } \overline z(\delta) - u).$$
As we suppose that   $x > \delta$, the inequality  
 $\sigma_{L}(u, x, y, z) \leq \delta$ becomes: 
$$\max(f(u, x), g(u, x, y, z)) =  \max((x + u - 1)^+, \dfrac{(x + u - 1+ y - z)^+}{2}) \leq \delta.$$ 
Let us put the inequality to be shown in the form: 
$$(x - \delta)^+ -  \max(0, y  \underset{L}{\longrightarrow} \overline z(\delta) - u)  \leq 0$$
and distinguish the two cases: $y - z \leq \delta$ and $y - z > \delta$.
\begin{enumerate}
    \item If   $y - z \leq \delta$, the inequality to be shown becomes: 
\begin{align}x - \delta  -  \max(0,   1 - u) & = x - \delta  + \min(0, u - 1)\nonumber\\
&= \min(x - \delta, x  + u - 1- \delta)   \nonumber\\
& \leq 0. \nonumber\end{align}
As  $x - \delta > 0$, we must show that  $ x  + u - 1- \delta \leq 0$, i.e., $x + u -1 \leq \delta$. \\But, we have:
  $$x + u -1  \leq (x + u -1)^+ = f(u, x) \leq \max(f(u, x), g(u, x, y, z))\leq \delta.$$
  \item If $y - z >  \delta$, the inequality to be shown becomes: 
\begin{align}x - \delta  -  \max(0,   1 - y + z +\delta - u) & = 
x - \delta  + \min(0, u - 1 + y - z   - \delta)\nonumber\\
&= \min(x - \delta, x +  u - 1  + y - z   -  2 \delta)   \nonumber\\
& \leq 0.\nonumber\end{align}
So, we must show that $x +  u - 1  + y - z   -  2 \delta \leq 0$, i.e 
$\dfrac{  x +  u - 1  + y - z }{2}  \leq \delta$. \\
But, we have:
$$\dfrac{ x +  u - 1  + y - z  }{2}   \leq \dfrac{  (x +  u - 1  + y - z)^+  }{2}  = g(u, x, y, z) 
\leq \max(f(u, x), g(u, x, y, z))\leq \delta.$$
\end{enumerate}
\noindent We have completed the proof of implication $\Longleftarrow $.

\end{proof}
From (Proposition \ref{prop4}), we immediately deduce:
\begin{theorem}\label{theoremedeLulu}
For any $x,y,z,u \in [0,1]$, we have:
    \begin{equation}\label{ineq:Lulu}
    \sigma_{L}(u, x, y, z) =    \min\{\delta\in [0, 1] \,\mid\, \underline{x}(\delta) \leq \max(0, \,y \underset{L}{\longrightarrow} \overline{z}(\delta) - u)  \}.
\end{equation}
\end{theorem}
\noindent We illustrate this result:
\begin{example}
Let $x = 0.4$,
    $y = 0.6$,
    $z = 0.3$ and 
    $u = 0.6$. We want  to obtain the smallest value of $\delta \in [0,1]$ so that $ \underline{x}(\delta) \leq \max(0, y \underset{L}{\longrightarrow} \overline{z}(\delta) - u)$ is true. \\
    We have $\max(0, y \underset{L}{\longrightarrow} z - u) = 0.1$ and $v = x+u-1 = 0$. 
    \begin{align*}
    \delta &\!\begin{aligned}[t]
    &= \sigma_{L}(u,x,y,z) \\
    &=  \min(0.4, \max(0, \frac{(0 + 0.6 - 0.3 )^+}{2}))\\
     &= 0.15.
     \end{aligned}
     \end{align*}
\noindent We have $\underline{x}(\delta)=x - 0.15 = 0.25$ and $\overline{z}(\delta)=z+0.15 = 0.45$. Therefore, $\max(0, y \underset{L}{\longrightarrow} \overline{z}(\delta) - u)= 0.25$ and:
\begin{equation*}
   \underline{x}(\delta) = \max(0, y \underset{L}{\longrightarrow} \overline{z}(\delta) - u).
\end{equation*}
\end{example}

\section{Analytical formulas to compute the Chebyshev distances associated to the second members of the systems \texorpdfstring{$A \Box_{T_P}^{\max} x = b$}{A max-product x = b} and \texorpdfstring{$A \Box_{T_L}^{\max} x = b$}{A max-Lukasiewicz x = b}}
\label{sec:formulas}
\noindent In this section, we compute  the Chebyshev distance  $\Delta$ (Definition \ref{def:chebyshevdist}) associated to the second member of a system of $\max$-product fuzzy relational equations $A \Box_{T_P}^{\max} x = b$, by an explicit analytical formula in terms  of components  of the matrix $A$  and those of the second member $b$. We also give an analytical formula  for the Chebyshev distance associated to the second member of a system of $\max$-Lukasiewicz fuzzy relational equations $A \Box_{T_L}^{\max} x = b$. \\
\noindent We rely on the equality:
$\Delta =        \min\{\delta\in [0, 1] \mid \underline b(\delta) \leq F(\overline b(\delta))\},$ see (\ref{eq:deltaCUNINGH}), which was proven for any continuous t-norm. 
\begin{notation}
\noindent Using a continuous t-norm $T$ whose associated implication operator is ${\cal I}_T$, we set for $1\leq i, k \leq n$ and  $1 \leq j \leq m$:
\begin{itemize}
    \item $K_i =  \{\delta \in [0, 1] \mid \underline b(\delta)_i \leq F(\overline b(\delta))_i \}$,
    \item $\beta_j  = \min_{1\leq   k \leq n}\,{\cal I}_T(a_{kj},  \overline b(\delta))_k)$,
    \item $W_{ij}^T = \{\delta \in [0, 1] \mid \underline b(\delta)_i \leq T(a_{ij}, \beta_j) \}$.
\end{itemize}
\end{notation}

\subsection{Analytical formula to compute the Chebyshev distance associated to the second member \texorpdfstring{$b$}{b} of the system \texorpdfstring{$A \Box_{T_P}^{\max} x = b$}{A max-product x = b}}

For any $i = 1, 2, \dots, n$ and $j = 1, 2, \dots ,m$, we have in this case: 
$$W_{ij}^{T_P}  = \{\delta \in [0, 1] \mid \underline b(\delta)_i \leq T_P(a_{ij}, \beta_j)\} = \{\delta \in [0, 1] \mid \underline b(\delta)_i \leq a_{ij}\beta_j\}.$$
where $\beta_j = \min_{1 \leq k \leq n}\, {\cal I}_{T_P}(a_{kj}, \overline{b}(\delta)_k) 
=  \min_{1 \leq k \leq n}\, a_{kj} \underset{GG}{\longrightarrow} \overline{b}(\delta)_k$.  
Then, we have: 
\begin{equation}\label{eq:deltap1}W_{ij}^{T_P}  =  
\{\delta \in [0, 1] \mid \underline b(\delta)_i \leq \min_{1 \leq k \leq n}\,a_{ij} \,  (a_{kj} \underset{GG}{\longrightarrow} \overline{b}(\delta)_k)\}.\end{equation} 
As $F(\overline b(\delta))_i = \max_{1 \leq j \leq m}\, T_P(a_{ij}, \beta_j) $, we still have: 
\begin{equation}\label{eq:deltap2}K_i  =  \bigcup_{1\leq j \leq m}\, W_{ij}^{T_P}.   \end{equation}
 
\noindent
The Chebyshev distance $\Delta$ associated to the second member $b$ of the system of $\max$-product fuzzy relational equations $A \Box_{T_P}^{\max} x = b$ is given by the following formula:
\begin{theorem}\label{th:2}
\begin{equation} \label{eq:deltap}
\Delta = \max_{1 \leq i \leq n}\,\delta_i
\end{equation}
where for $i = 1, 2, \dots, n$: 
\begin{equation} \label{eq:DeltapiGG} 
\delta_i =  \min_{1 \leq j \leq m}\, \max_{1 \leq k \leq n}\,\sigma_{GG}\,(a_{ij}, b_i, a_{kj}, b_k).
\end{equation}
\noindent See (\ref{eq:sigmagg}) for the definition of $\sigma_{GG}$.
\end{theorem}
\begin{proof} For any $i = 1, 2, \dots, n$ and $j = 1, 2, \dots, m$, we deduce from (\ref{eq:deltap1}) that for  
$\delta\in [0, 1]$, we have: 
$$\delta \in W_{ij}^{T_P}  
\,\Longleftrightarrow\,
 \delta \geq \,\max_{1 \leq k \leq n}\,\sigma_{GG}\,(a_{ij}, b_i, a_{kj}, b_k).$$
Using (\ref{eq:deltap2}), we get:
$$\delta \in K_i 
\,\Longleftrightarrow\,
\exists \, j\in\{1, 2, \dots, m\} \,\, \text{such that}\,\, \delta \geq \max_{1 \leq k \leq n}\,\sigma_{GG}\,(a_{ij}, b_i, a_{kj}, b_k).
$$
So, we obtain: 
$$\delta \in K_i 
\,\Longleftrightarrow\, 
 \delta \geq \min_{1 \leq j \leq m}\,  \max_{1 \leq k \leq n}\,\sigma_{GG}\,(a_{ij}, b_i, a_{kj}, b_k).$$  
As by definition $\delta \in K_i    
\,\Longleftrightarrow\,
\underline b(\delta)_i \leq F(\overline b(\delta))_i$ and $\Delta =         \min\{\delta\in [0, 1] \mid \underline b(\delta) \leq F(\overline b(\delta))\} $, we get: 
$$\Delta = \max_{1 \leq i \leq n}\, \min_{1 \leq j \leq m}\,  \max_{1 \leq k \leq n}\,\sigma_{GG}\,(a_{ij}, b_i, a_{kj}, b_k).$$
\end{proof}

\noindent The following example illustrates this result: 
\begin{example}
In what follows, we rely on the following matrix and vector, which were used as example in \citep{cuninghame1995residuation,pedrycz1990inverse}:

\begin{equation}\label{eq:Abofpedrycz}
A = \begin{bmatrix}
    1 & 0.4 & 0.5 & 0.7\\
    0.7 & 0.5 & 0.3 & 0.5\\
    0.2 & 1 & 1 & 0.6\\ 
    0.4 & 0.5 & 0.5 & 0.8
\end{bmatrix}\text{ and } b = \begin{bmatrix}
        0.4\\ 1\\ 0.2\\ 0
    \end{bmatrix}.
\end{equation}

\noindent For the system $A \Box_{T_P}^{\max} x = b$ which uses the t-norm product $T_P$, we compute from (Theorem \ref{th:2}):

\begin{align*}
    \delta_1 &=  \min_{1 \leq j \leq m}\, \max_{1 \leq k \leq n}\,\sigma_{GG}\,(a_{1j}, b_1, a_{kj}, b_k) = \min(0.11, 0.23, 0.2, 0.21) = 0.11,\\
     \delta_2 &=  \min_{1 \leq j \leq m}\, \max_{1 \leq k \leq n}\,\sigma_{GG}\,(a_{2j}, b_2, a_{kj}, b_k) = \min(0.42, 0.6, 0.72, 0.62) = 0.42,\\
     \delta_3 &=  \min_{1 \leq j \leq m}\, \max_{1 \leq k \leq n}\,\sigma_{GG}\,(a_{3j}, b_3, a_{kj}, b_k) = \min(0.13, 0.07, 0.07, 0.11) = 0.07,\\
     \delta_4 &=  \min_{1 \leq j \leq m}\, \max_{1 \leq k \leq n}\,\sigma_{GG}\,(a_{4j}, b_4, a_{kj}, b_k) = \min(0.0,0.0,0.0,0.0) = 0.0.\\
\end{align*}

\noindent Therefore, the Chebyshev distance associated to the second member $b$ of the system $A \Box_{T_P}^{\max} x = b$ is $\Delta = \max(\delta_1,\delta_2,\delta_3,\delta_4) = 0.42$.

We compute $\overline{b}(\Delta)=\begin{bmatrix}
    0.82\\ 1\\ 0.62\\ 0.42
\end{bmatrix}$ and $F(\overline{b}(\Delta))= \begin{bmatrix}
    0.82\\ 0.57\\ 0.62\\ 0.42
\end{bmatrix}$ is the greatest Chebyshev approximation of the second member~$b$. 
\end{example}

\subsection{Analytical formula to compute the Chebyshev distance associated to the second member \texorpdfstring{$b$}{b} of the system \texorpdfstring{$A \Box_{T_L}^{\max} x = b$}{A max-Lukasiewicz x = b}}

For any $i = 1, 2, \dots, n$ and $j = 1, 2, \dots ,m$, we have in this case: 
$$W_{ij}^{T_L}  = \{\delta \in [0, 1] \mid \underline b(\delta)_i \leq T_L(a_{ij}, \beta_j)\} = \{\delta \in [0, 1] \mid \underline b(\delta)_i \leq \max(0,\,\, a_{ij} + \beta_j - 1)\}.$$
where $\beta_j = \min_{1 \leq k \leq n}\, {\cal I}_{T_L}(a_{kj}, \overline{b}(\delta)_k) 
=  \min_{1 \leq k \leq n}\, a_{kj} \underset{L}{\longrightarrow} \overline{b}(\delta)_k$.

\noindent
Set $u_{ij} = 1 - a_{ij} $. Then, we have:  
 \begin{align}
T_L(a_{ij}, \beta_j) & =  \max (0,\,\,\beta_j  - u_{ij})\nonumber\\
& = \max(0,\,\,[\min_{1 \leq k \leq n}\, a_{kj} \underset{L}{\longrightarrow} \overline{b}(\delta)_k] -  u_{ij})\nonumber\\
& = \min_{1 \leq k \leq n}\, \max(0,\,\,[a_{kj} \underset{L}{\longrightarrow} \overline{b}(\delta)_k]  -  u_{ij})\nonumber\\
& = \min_{1 \leq k \leq n}\,  ([a_{kj} \underset{L}{\longrightarrow} \overline{b}(\delta)_k] - u_{ij})^+.  \nonumber\end{align}
So, we obtain: 
\begin{equation}\label{eq:deltaL1}W_{ij}^{T_L}  =  
\{\delta \in [0, 1] \mid \underline b(\delta)_i \leq \min_{1 \leq k \leq n}\,  ([a_{kj} \underset{L}{\longrightarrow} \overline{b}(\delta)_k] - u_{ij})^+\}.\end{equation} 
As $F(\overline b(\delta))_i = \max_{1 \leq j \leq m}\, T_L(a_{ij}, \beta_j) $, we still have: 
\begin{equation}\label{eq:deltaL2}K_i  =  \bigcup_{1\leq j \leq m}\, W_{ij}^{T_L}.    \end{equation}
 
\noindent
The Chebyshev distance $\Delta$ associated to the second member $b$ of the system of $\max$-Lukasiewicz fuzzy relational equations $A \Box_{T_L}^{\max} x = b$ is given by the following formula:
\begin{theorem}\label{th:3}
\begin{equation} \label{eq:DeltaL}
\Delta = \max_{1 \leq i \leq n}\,\delta_i
\end{equation}
where for any $i = 1, 2, \dots n$: 
\begin{equation} \label{eq:DeltapiL} 
\delta_i =  \min_{1 \leq j \leq m}\, \max_{1 \leq k \leq n}\,\sigma_{L }\,(1  - a_{ij}, b_i, a_{kj}, b_k).
\end{equation}
\noindent See (\ref{eq:sigmaL}) for the definition of $\sigma_{L}$.
\end{theorem}
\begin{proof} For any $i = 1, 2, \dots, n$ and $j = 1, 2, \dots, m$, we deduce from (\ref{eq:deltaL1}) that for  
$\delta\in [0, 1]$, we have: 
$$\delta \in W_{ij}^{T_L}  
\,\Longleftrightarrow\,
 \delta \geq \,\max_{1 \leq k \leq n}\,\sigma_{LK}\,(1 - a_{ij}, b_i, a_{kj}, b_k).$$
Using (\ref{eq:deltaL2}), we get:
$$\delta \in K_i 
\,\Longleftrightarrow\,
\exists \, j\in\{1, 2, \dots, m\} \,\, \text{such that}\,\, \delta \geq \max_{1 \leq k \leq n}\,\sigma_{LK}\,(1 - a_{ij}, b_i, a_{kj}, b_k).
$$
So, we obtain: 
$$\delta \in K_i 
\,\Longleftrightarrow\, 
 \delta \geq \min_{1 \leq j \leq m}\, \max_{1 \leq k \leq n}\,\sigma_{LK}\,(1 - a_{ij}, b_i, a_{kj}, b_k).$$  
As by definition $\delta \in K_i    
\,\Longleftrightarrow\,
\underline b(\delta)_i \leq F(\overline b(\delta))_i$ and $\Delta =         \min\{\delta\in [0, 1] \mid \underline b(\delta) \leq F(\overline b(\delta))\} $, we get: 
$$\Delta = \max_{1 \leq i \leq n}\, \min_{1 \leq j \leq m}\,  \max_{1 \leq k \leq n}\,\sigma_{LK}\,(1 - a_{ij}, b_i, a_{kj}, b_k).$$
\end{proof}  
 
\begin{example}

 We rely on the matrix $A$ and the vector $b$ used in the previous example, see (\ref{eq:Abofpedrycz}), to construct the system $A \Box_{T_L}^{\max} x = b$, which uses  with the Lukasiewicz's t-norm  $T_L$. From (Theorem \ref{th:3}), we have:
\begin{align*}
    \delta_1 &=  \min_{1 \leq j \leq m}\, \max_{1 \leq k \leq n}\,\sigma_{L }\,(1  - a_{1j}, b_1, a_{kj}, b_k) = \min(0, 0.4, 0.35, 0.25) = 0,\\
    \delta_2 &=  \min_{1 \leq j \leq m}\, \max_{1 \leq k \leq n}\,\sigma_{L }\,(1  - a_{2j}, b_2, a_{kj}, b_k) = \min(0.45, 0.65, 0.75, 0.65) = 0.45,\\
    \delta_3 &=  \min_{1 \leq j \leq m}\, \max_{1 \leq k \leq n}\,\sigma_{L }\,(1  - a_{3j}, b_3, a_{kj}, b_k) = \min(0.2, 0, 0, 0.2) = 0,\\
    \delta_4 &=  \min_{1 \leq j \leq m}\, \max_{1 \leq k \leq n}\,\sigma_{L }\,(1  - a_{4j}, b_4, a_{kj}, b_k) = \min(0,0,0,0) = 0.\\
\end{align*}
\noindent Therefore, the Chebyshev distance associated to the second member $b$ of the system $A \Box_{T_L}^{\max} x = b$ is $\Delta = \max(\delta_1,\delta_2,\delta_3,\delta_4) = 0.45$.

We compute $\overline{b}(\Delta)=\begin{bmatrix}
    0.85 \\ 1\\ 0.65\\ 0.45
\end{bmatrix}$ and $F(\overline{b}(\Delta))= \begin{bmatrix}
    0.85\\ 0.55\\ 0.65\\ 0.45
\end{bmatrix}$ is the greatest Chebyshev approximation of the second member~$b$. 

\end{example}

\section{Conclusion}

In this article, we gave explicit analytical formulas for computing  the Chebyshev distance associated to the second member of a system of $\max$-product fuzzy relational equations and that associated to the second member of a system of $\max$-Lukasiewicz fuzzy relational equations.

The results of \citep{arxiv.2301.06141} may be then extended to these two systems. For each system, we can describe the sets of Chebyshev approximations of its second member and its approximate solutions set. Similarly to the   $\max-\min$ learning paradigm introduced in \citep{arxiv.2301.06141}, in which the learning error is expressed in terms of $L_\infty$ norm, we can tackle the development of a $\max$-product (resp. $\max$-Lukasiewicz) learning paradigm. The formula of \citep{arxiv.2301.06141}, which computes the minimum value of the learning error according to  training data and the method to construct approximate weight matrices whose learning error is minimal can be directly extended to the case of the $\max$-product composition or the  $\max$-Lukasiewicz composition. 

As applications, we have the complete solution of the problem of the invertibility of a fuzzy relation for $\max$-Lukasiewicz composition: we know the set of matrices $A$ which admit a pre-inverse or a post-inverse. Moreover, for $\max-\min$ composition, if a matrix $A$ has no preinverse (resp. postinverse), we know how to compute, using  the $L_\infty$ norm, an approximate preinverse  (resp. postinverse) for $A$. 
For both cases of the $\max-\min$ composition and the $\max$-product composition, the problem of the invertibility of a fuzzy relation is already solved \citep{wu2021analytical,wu2022analytical}.  

\bibliographystyle{plainnat}

\end{document}